\documentclass[letterpaper, 10 pt, conference]{ieeeconf}  

\IEEEoverridecommandlockouts                              

\usepackage{graphicx}
\usepackage{epsfig} 
\usepackage{mathptmx}
\usepackage{amsmath} 

\usepackage{amssymb}  
\usepackage{amsthm}
\usepackage{wasysym}
\usepackage{mathrsfs}
\usepackage[mathcal]{euscript}
\usepackage{bbm}
\usepackage{color}
\usepackage{lipsum}
\usepackage[ruled,vlined,linesnumbered]{algorithm2e}\usepackage[colorlinks=true,linkcolor=black,anchorcolor=black,citecolor=black,filecolor=black,menucolor=black,runcolor=black,urlcolor=magenta]{hyperref}
\usepackage{etoolbox}
\usepackage[table,xcdraw,dvipsnames]{xcolor}
\usepackage{multirow}
\usepackage[normalem]{ulem}
\usepackage{xcolor}
\usepackage{tabularx}
\usepackage{booktabs}
\usepackage{array}
\newcolumntype{Y}{>{\centering\arraybackslash}X}

\usepackage{outlines}
\let\labelindent\relax
\usepackage{enumitem}
\setenumerate[2]{label=\alph*.}
\setenumerate[3]{label=\roman*.}

\makeatletter
\patchcmd{\@makecaption}
  {\scshape}
  {}
  {}
  {}
\makeatother

\newtheorem{defn}{Definition}
\newtheorem{rem}[defn]{Remark}
\newtheorem{lem}[defn]{Lemma}

\newtheorem{thm}[defn]{Theorem}
\newtheorem{cor}[defn]{Corollary}
\newtheorem{problem}[defn]{Problem}

\providecommand{\R}{\ensuremath \mathbb{R}}
\providecommand{\N}{\ensuremath \mathbb{N}}
\newcommand{\unitcircle}{\regtext{SO}(1)}
\newcommand{\regtext}[1]{\mathrm{\textnormal{#1}}}

\newcommand{\tss}[1]{\textsuperscript{#1}}
\newcommand{\lbl}[1]{_{\regtext{#1}}}

\newcommand{\vc}[1]{#1}
\newcommand{\set}[1]{\mathcal{#1}}
\newcommand{\setfn}[1]{\mathscr{#1}}

\newcommand{\eye}{\vc{I}}

\newcommand{\rotmat}{\vc{R}}
\newcommand{\rotaxis}{\vc{w}}
\newcommand{\transvec}{\vc{t}}

\newcommand{\union}{\bigcup}

\newcommand{\given}{\mid}
\newcommand{\supercross}{^\times}

\newcommand{\linesegment}[1]{\overline{#1}}
\DeclareMathOperator{\boxset}{box}
\newcommand{\rotset}{\set{R}}
\newcommand{\transset}{\set{T}}
\newcommand{\dataset}{\set{D}}

\newcommand{\link}{\set{L}}

\newcommand{\final}{\lbl{f}}

\newcommand{\timestep}{t}
\newcommand{\horizon}{T}
\newcommand{\planninghorizon}{{\horizon_{\plan}}}
\newcommand{\backuphorizon}{{\horizon_{b}}}
\newcommand{\actionhorizon}{{\horizon_{\action}}}

\newcommand{\state}{\vc{x}}
\newcommand{\action}{\vc{a}}

\newcommand{\observation}{\vc{o}}
\newcommand{\plan}{\vc{p}}
\newcommand{\param}{\vc{k}}

\newcommand{\config}{\vc{q}}
\newcommand{\ucpoint}{\vc{p}}
\newcommand{\jointangle}{\theta}
\newcommand{\jointset}{\Theta}

\newcommand{\statespace}{\set{X}}
\newcommand{\workspace}{\set{W}}

\newcommand{\actionspace}{\set{A}}
\newcommand{\configspace}{\set{Q}}
\newcommand{\obstacles}{\workspace\lbl{obs}}

\newcommand{\forwardoccupancy}{\regtext{FO}}
\newcommand{\forwardoccapprox}{\widetilde{\forwardoccupancy}}
\newcommand{\observationspace}{\set{O}}

\newcommand{\dynamics}{\vc{f}}
\DeclareMathOperator{\rotfunc}{RO}
\newcommand{\policy}{\vc{\pi}}

\newcommand{\zonotope}{\setfn{Z}}

\newcommand{\jth}{$j$\tss{th}\xspace}

\newcommand{\st}{\regtext{ s.t. }}
\newcommand{\backup}{\regtext{back}}
\DeclareMathOperator*{\argmin}{argmin}

\usepackage[
    style=ieee,
    doi=false,
    isbn=false,
    url=false,
    eprint=false,
    backend=biber,
    natbib=true
    ]{biblatex}
    
\bibliography{references}

\title{\LARGE \bf
RAIL: Reachability-Aided Imitation Learning for Safe Policy Execution
}

\author{Wonsuhk Jung, Dennis Anthony, Utkarsh A. Mishra, Nadun Ranawaka Arachchige, Matthew Bronars \\ Danfei Xu$^*$, and Shreyas Kousik$^*$
\thanks{* indicates equal contribution}
\thanks{This work was supported by the Georgia Tech AMPF.
All authors are with the Georgia Institute of Technology, Atlanta, GA, USA.
Corresponding author: \texttt{wonsuhk.jung@gatech.edu}}}

\begin{document}

\newcommand\redsout{\bgroup\markoverwith{\textcolor{red}{\rule[0.5ex]{2pt}{0.4pt}}}\ULon}

\maketitle
\thispagestyle{plain}
\pagestyle{plain} 

\begin{abstract}
Imitation learning (IL) has shown great success in learning complex robot manipulation tasks.
However, there remains a need for practical safety methods to justify widespread deployment.
In particular, it is important to certify that a system obeys hard constraints on unsafe behavior in settings when it is unacceptable to design a tradeoff between performance and safety via tuning the policy (i.e. soft constraints).
This leads to the question, how does enforcing hard constraints impact the performance (meaning safely completing tasks) of an IL policy?
To answer this question, this paper builds a reachability-based safety filter to enforce hard constraints on IL, which we call Reachability-Aided Imitation Learning (RAIL).
Through evaluations with state-of-the-art IL policies in mobile robots and manipulation tasks, we make two key findings.
First, the highest-performing policies are sometimes only so because they frequently violate constraints, and significantly lose performance under hard constraints.
Second, surprisingly, hard constraints on the lower-performing policies can occasionally increase their ability to perform tasks safely.
Finally, hardware evaluation confirms the method can operate in real time.
More results can be found at our website: \href{https://rail2024.github.io/}{https://rail2024.github.io/}.
\end{abstract}
\section{Introduction} \label{sec:intro}
Model-free learning has shown great success in the robotics domain, enabling robots to acquire complex behaviors without explicit modeling of task dynamics. 
Among these approaches, offline imitation learning (IL) has emerged as a powerful paradigm, allowing robots to learn from demonstrations without the need for extensive environmental interactions or carefully designed reward functions. Recent advancements in IL have demonstrated impressive generalization \cite{padalkar2023open}, performance \cite{shafiullah2022behavior}, and the ability to generate temporally consistent plans \cite{chi2023diffusion}.  In this work, we aim to build on top of effective IL frameworks.

That said, IL models still face significant challenges when deployed in real-world scenarios, particularly in safety-critical applications where performance must consider both task success and adherence to safety specifications.
These challenges stem from several factors.
The absence of physical interaction during training can lead to compounding errors when executing planned actions \cite{bagnell2015imitation}.
IL models tend to be brittle when encountering scenarios beyond their training data \cite{argall2009survey, bagnell2015imitation}.
Pre-trained models may struggle to adapt to dynamic changes in the environment during inference.
In the worst case, these factors can cause a robot to act unsafely by colliding with its environment and causing damage.
However, if we were to enforce hard safety constraints on a policy, it stands to reason that we may simultaneously sacrifice performance.

\begin{figure}[t]
    \centering
    \includegraphics[width=0.9\linewidth]{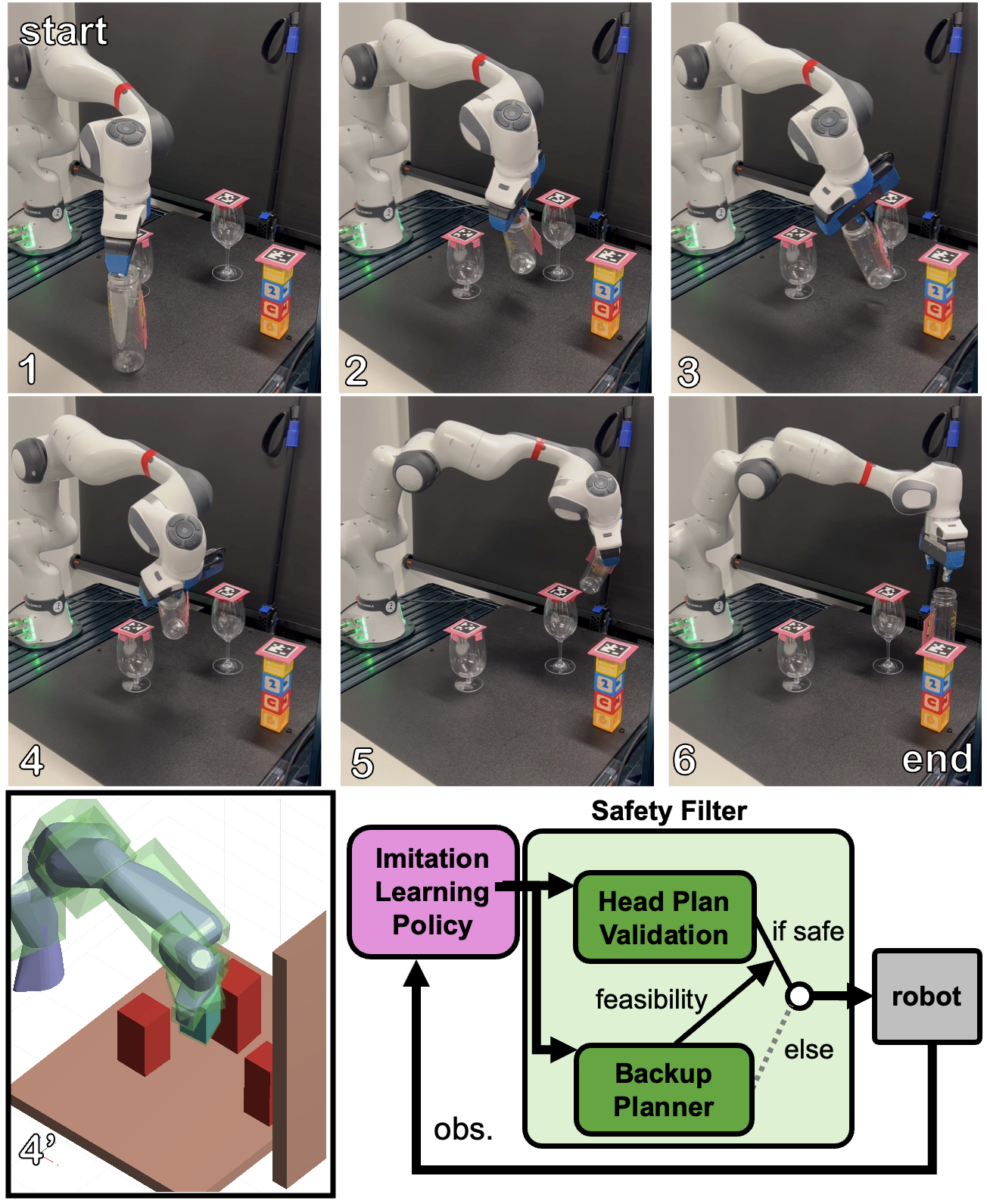}
    \caption{Our RAIL framework applied to a real-world robotic manipulation task. \textbf{Top}: A Franka robot arm safely executes a pick-and-place task among delicate obstacles.
    \textbf{Bottom}: System diagram illustrating how RAIL integrates an imitation learning policy with a safety filter, using plan validation and a failsafe planner to enforce hard constraints.
    }
    \label{fig:teaser_figure}
    \vspace*{-0.4cm}
\end{figure}

This paper seeks to understand the potential impact of imposing hard constraints on an IL policy.
Doing so first requires a strategy for enforcing hard constraints on IL.
To this end, we adapt techniques from model-based robot motion planning to propose \textbf{R}eachability-\textbf{A}ided \textbf{IL} (\textbf{RAIL}). 
Our approach introduces a safety filter that acts as a supervisory layer over the learned policy
by leveraging recent advancements in continuous-time collision checking and reachability analysis to provide robust safety guarantees in a three-dimensional workspace.
To achieve this in a robotics manipulation task, one critical challenge is efficiently computing the continuous-time swept volume of a multi-link manipulator arm given the planned action sequence output by the motion policy, such that the policy can run in real-time.
We address this challenge with a novel extension of prior work on polynomial zonotopes \cite{holmes2020armtd,michaux2023can} to overapproximate the robot's swept volume for fast, accurate collision checking in continuous time. 

By integrating this safety filter with IL policies, we find that our method allows for the exploitation of the strengths of learned policies in generating complex behaviors while enforcing hard constraints over an infinite time horizon.
Importantly, RAIL is composable with different state-of-the-art IL frameworks, including Diffusion Policy~\cite{chi2023diffusion} and ACT~\cite{zhao2023learning}, suggesting broad applicability. 

We evaluate the impact of hard constraints on IL via a series of increasingly complex tasks: navigating a 2D maze in simulation, manipulating objects in clutter in simulation, and maneuvering Franka arm hardware to transport a bottle while avoiding collisions with wine glasses and stacked blocks.
As expected, RAIL significantly improves policy safety across all tasks, achieving 0\% collision rates compared to 5-35\% for baseline policies, while maintaining high task success rates.
We find that the highest-performing seeds from these baseline policies also suffer a performance drop under hard constraints.
However, RAIL, in some cases excitingly, aids IL policies in discovering new behavioral modalities that safely complete the task, even when the IL policies alone fail to do so.
This demonstrates that incorporating hard safety constraints can lead to \textit{increased performance}, challenging the notion that they are inherently in opposition.
\section{Related Work}
\textbf{Offline Imitation Learning.}
Offline imitation learning (IL) algorithms typically learn from demonstrations through supervised regression~\cite{pomerleau1988alvinn, zhang2018deep}, offline reinforcement learning~\cite{levine2020offline, kumar2020conservative, janner2021offline, chebotar2023q}, or generative models~\cite{pearce2023imitating, mishra2023reorientdiff, mishra2023generative, mishra2024generative, reuss2023goal, reuss2023multimodal, janner2022planning, ajay2022conditional, chi2023diffusion}. Notably, Diffusion Policy~\cite{chi2023diffusion} introduces a powerful diffusion-based framework for offline IL. It learns the conditional distribution of action trajectory given the history of observations and executes the planned action sequence in a receding horizon fashion.  While such policies can learn complex behaviors effectively, their practical deployment presents several challenges: (1) the absence of physical interaction results in compounding errors while planning~\cite{bagnell2015imitation}, and (2) these methods tend to be fragile when encountering scenarios beyond their training data~\cite{argall2009survey, bagnell2015imitation}. Consequently, such policies can be unsafe. 
To address these issues, our method proposes a safety filter \cite{wabersich2023data, hsu2023safety} for offline IL policies to exploit the strengths of these policies while guaranteeing safety. 

\textbf{Safety for Learned Policies.}
A wide variety of approaches seek to enforce safety with learned policies.
Many approximately enforce collision avoidance constraints (e.g., \cite{dawson2022safe,qin2022sablas,carvalho2022conditioned,carvalho2023motion, ciftci2024safe}) or more complex constraints (e.g., staying close to contact \cite{kicki2023fast}); broadly speaking, the unifying concept is to incorporate a safety specification as a penalty or loss in training or test-time inference.
Other methods focus on probabilistic safety constraints \cite{menda2019ensembledagger, castaneda2022probabilistic}, ensuring that the likelihood of safety violations remains below a specified threshold.
In contrast, our method focuses on enforcing hard constraints as seen in  \cite{ yin2021imitation, chen2019deep, xiao2023safediffuser, bouvier2024policedrl, thumm2022provably,krasowski2020safe}, including our prior work \cite{selim2022safe,shao2021reachability}.
The common trend across these is to analyze \textit{invariant sets} \cite{chen2019deep,xiao2023safediffuser, bouvier2024policedrl} or \textit{reachable sets} \cite{krasowski2020safe, thumm2022provably, selim2022safe, shao2021reachability} of the trajectories of an uncertain dynamical system, and develop strict fallback or failsafe policies that can replace a learned policy's actions (i.e., a safety filter).
The key challenge across all such safety methods is to balance safety against performance, which requires careful design of the safety filter.
While some hard constraint methods for reinforcement learning see a drop or no change in performance \cite{thumm2022provably,krasowski2020safe}, others see an increase \cite{shao2021reachability}, suggesting the same may be true for IL.

\textbf{Safety for Manipulators.}
Manipulator collision avoidance is a classical problem that is challenging due to the nonlinear, set-valued map between state space and workspace \cite{lavalle2006planning}, and due to the need to represent collision avoidance in continuous time \cite{redon2005fast}.
Recent progress in numerical representation has enabled new ways to compute and differentiate through continuous-time collision checks with polynomial zonotopes \cite{holmes2020armtd,michaux2023can,michaux2024safe}, capsules \cite{schepp2022sara,pereira2017calculating}, or learned collision functions \cite{michaux2023rdf,chiang2021fast}.
A further challenge is to consider uncertain dynamics, which can be handled with polynomial zonotopes \cite{michaux2023can} or control barrier functions \cite{singletary2022safety}.
In this space, we extend our recent work \cite{holmes2020armtd,michaux2023can} from parameterized trajectories to trajectories planned by an IL policy.

\section{Problem Statement and Preliminaries}

To understand the impact of hard constraints on IL, we must first address the following:
\begin{problem}
    Plan and control the motion of a robot to complete tasks while guaranteeing the robot avoids collision with obstacles in continuous time as a hard constraint.
\end{problem}

\begin{rem}
    While we define safety as avoiding collisions between the robot and obstacles, the method can flexibly incorporate other safety specifications such as position or velocity constraints, as demonstrated in Section \ref{sec:results}.
\end{rem}

\textbf{Robot, Environment, and Safety.}
We consider a robot with state $\state(t) \in \statespace \subseteq \R^n$ at time $t$ and action $\action(t) \in \actionspace$ (e.g., desired end effector poses), executed by the low-level controller (e.g., Operational Space Controller \cite{khatib1987unified}.)
The robot's dynamics are represented as $\dot{\state}(t) = \dynamics(\state(t),\action(t))$.
The robot occupies a workspace $\workspace \subset \R^3$ via a forward occupancy map $\forwardoccupancy(\state(t)) \subset \workspace$.
We denote workspace obstacles $\obstacles \subset \workspace$.
The robot receives observations $\observation(t) \in \observationspace$ (e.g., joint states or RGB-D images) at a fixed rate.
We assume full observability, where observations $\observation(t)$ include privileged information for safety verification, such as the true robot state, task-relevant object poses, and object physics.
Given a trajectory $\state: [0,\infty) \to \configspace$, safety means $\forwardoccupancy(\state(t)) \cap \obstacles = \emptyset$ for all time, plus other task-dependent constraints (e.g., max velocity).

\textbf{Set Operations.}
We use the following operations to compute swept volumes to enforce hard constraints.
For $\vc{p}_1, \vc{p}_2 \in \R^n$, the line segment between them is
\begin{align*}
    \linesegment{\vc{p}_1\vc{p}_2} =
    \big\{\vc{p} \in \R^n \given
        \vc{p} = \alpha\vc{p}_1 + (1-\alpha)\vc{p}_2,\ \alpha \in [0,1]
    \big\}.
\end{align*}
Let $\set{X}, \set{Y} \subset \R^n$.
We use the Minkowski sum, Cartesian product, set product, and cross product (for $n = 3$):
\begin{align*}
    \set{X} \oplus \set{Y} &= \left\{ 
        \vc{x} + \vc{y} \given
        \vc{x} \in \set{X}, \vc{y} \in \set{Y}
    \right\}, \\
    \set{X}\times\set{Y} &= \{
            (\vc{x},\vc{y}) \in \R^{n+m} \given
             \vc{x} \in \set{X}, \vc{y} \in \set{Y}
        \}, \\
    \set{X}\set{Y} &= \left\{ 
        \vc{X}\vc{Y} \in \R^{n\times k} \given \vc{X} \in \set{X}, \vc{Y} \in \set{Y}
    \right\},\ \regtext{and} \\
    \set{X}\supercross\set{Y} &= \left\{
        \left[\begin{smallmatrix}
            0 & -x_3 & x_2 \\
            x_3 & 0 & -x_1 \\
           -x_2 & x_1 & 0
        \end{smallmatrix}\right]\vc{y} \in \R^3 \given
        \left[\begin{smallmatrix}
            x_1 \\ x_2 \\ x_3
        \end{smallmatrix}\right] \in \set{X},\ 
        \vc{y} \in \set{Y}
    \right\}.
\end{align*}
For a function $f$, we denote $f(\set{X}) = \{f(\vc{x}) \given \vc{x} \in \set{X}\}$.
We apply linear operations on the left or right as appropriate (e.g., $\vc{A}\set{X}$ or $\set{X}\vc{A}$).
We numerically implement all set operations with polynomial zonotopes \cite{kochdumper2020sparse} as per \cite[\S III-C]{michaux2023can}.
\section{Framework Overview}

We aim to incorporate hard safety guarantees in state-of-the-art imitation learning frameworks directly at inference.
Our method adopts reachability-based trajectory design and continuous time collision checking to guarantee safety while leveraging the multi-modal behavior learning of a baseline motion policy.
We provide the details for each design segment of our framework below.

\subsection{Motion Policy}
We consider the paradigm of offline IL where we have a dataset $\dataset = \{(o, a)\}$ containing pre-recorded demonstration trajectories (sequences of observation-action pairs), of successfully completing the task objective. 
We seek to train an imitation learning policy that maps observations to action motion plans $\policy: \observationspace \to \actionspace^\planninghorizon$ where $\planninghorizon \in \N$  is the plan length.
Finally, we assume the obstacles $\obstacles$ may not be present in the training data, or may change from training to test time, but can be observed at test time. In this paper, we mainly consider Diffusion Policy~\cite{chi2023diffusion} as our motion policy and learn the target conditional distribution $p_\pi(\action_{\timestep:\timestep+\planninghorizon}|\observation_t)$. We note that our method can generalize to any similar offline IL algorithms and include additional base algorithms in experiments. 

\subsection{Reachability-Assisted Imitation Learning}

\begin{algorithm}[t]
\label{algo:rail}
\small
\caption{Reachability Aided IL (RAIL)}
\DontPrintSemicolon
\KwIn{Observation $\observation_{0}$, IL policy $\policy$, Backup Policy $\policy_{\backup}$}

// initialize plan and backup plan \;
$(\timestep, \state_{0}) \gets (0, \text{ \textbf{stateEstimate}}(\observation_0))$\;
$\hat \action_{0:\backuphorizon} \gets \policy_{\backup}(\state_{0})$\;
$(\action_{0:\planninghorizon}, {}^{\backup}\action) \gets (\hat \action_{0:\planninghorizon}, \hat{\action}_{\actionhorizon: \backuphorizon})$ \;
\While{task is not done}{
Apply $\action_{\timestep:\timestep+\actionhorizon}$ to the robot for horizon of $\timestep \text{ to } \timestep + \actionhorizon$\;
// execute lines below while applying actions \;
$\observation_{\timestep + \actionhorizon} \gets \text{\textbf{predictObservation}}(\observation_{\timestep}, \action_{\timestep:\timestep+\actionhorizon})$\;
$\state_{\timestep + \actionhorizon} \gets \text{\textbf{stateEstimate}}(\observation_{\timestep + \actionhorizon})$

$\action_{\timestep+\actionhorizon:\timestep+\actionhorizon+\planninghorizon} \gets \policy(\observation_{\timestep + \actionhorizon})$\; 
$(\action_{\timestep+\actionhorizon: \timestep+2\actionhorizon}, {}^{\backup}\action) \gets \text{Filter}(\state_{\timestep+\actionhorizon}, \action_{\timestep+\actionhorizon:\timestep+\actionhorizon+\planninghorizon}, {}^{\backup}\action$) \;

\If{${}^{\backup}\action$ is empty}{
// robot stopped safely \;
Set up a backup plan as Line 2-4}
}
\end{algorithm}

\begin{algorithm}[t]
\label{algo:reach_filter}
\small
\caption{Filter}
\DontPrintSemicolon

\KwIn{State $\state_{\timestep}$, Nominal plan $\action_{\timestep:\timestep+\planninghorizon}$ Backup plan ${}^{\backup}\action$}
\KwOut{Filtered plan $\action_{\timestep:\timestep+\actionhorizon}$, Updated Backup plan ${}^{\backup}\action$}

$\state_{\timestep:\timestep+\planninghorizon} \gets \textbf{predictState}(\state_{\timestep}, \action_{\timestep:\timestep+\planninghorizon})$ \;

\If{$\forwardoccupancy(\state_{\timestep:\timestep+\actionhorizon}) \cap \obstacles = \emptyset$}{
    $\hat \action_{\timestep+\actionhorizon:\timestep+\actionhorizon+\backuphorizon} \gets \policy_{\backup}(\state_{\timestep+\actionhorizon})$\;
    \If{$\hat \action_{\timestep+\actionhorizon:\timestep+\actionhorizon+\backuphorizon}$ is feasible}{
        \Return $\action_{\timestep:\timestep+\actionhorizon}$, $\hat{\action}_{\timestep+\actionhorizon:\timestep+\actionhorizon+\backuphorizon}$ \;
    }
    \Else{
        \Return ${}^{\backup}\action_{\timestep:\timestep+\actionhorizon}$,${}^{\backup}\action_{\timestep+\actionhorizon:\text{end}}$ \;
    }
}
\end{algorithm}

This section presents our first contribution: a framework that interfaces an imitation learning policy with a model-based backup planner using reachability analysis.

We adopt receding-horizon planning, where the robot plans over a finite horizon $\planninghorizon$, but only executes the first $\actionhorizon$ steps before re-planning. 

The key idea is to maintain a safe failsafe \cite{thumm2022provably} backup plan using reachability analysis while verifying the safety of the proposed plan by imitation learning policy.
Towards, we build upon the safety filter concept from \cite{wabersich2023data, hsu2023safety}, where a safe policy intervenes only when the nominal policy is verified to be unsafe.
We especially override the definition of safety with not only the safety of the plan proposed by the imitation learning policy but also verifying whether the backup plan can be attached at the end of the plans to execute (i.e., $\action_{\timestep: \timestep+\actionhorizon}.)$
The framework comprises two key safety verification algorithms for the proposed plan $\action_{\timestep: \timestep + \planninghorizon}$

\textbf{Head Plan Verification.}
First, we introduce the head-plan safety verifier, which checks the continuous-time safety of arbitrary plans from the imitation learning policy using reachability analysis.
We verify the safety of the plan $\action_{\timestep: \timestep + \actionhorizon}$, which we call as \textit{head-plan}.
We detail this verifier, particularly for manipulators, in Section \ref{sec:cont_time_col_check}.

\textbf{Backup Plan Creation.}
Second, we verify the existence of a backup plan, using a reachability-based model-based planner similar to \cite{kousik2017safe, holmes2020armtd}.
This planner generates failsafe plans, parameterized by low-dimensional (e.g., number of joints for a manipulator) parameter $\param$, where a failsafe maneuver is a stopping action.
Note that the failsafe is not concerned with task completion, only safety, so using a low-dimensional parameter is sufficient.

At each iteration, the planner solves the following nonlinear optimization problem to find a safe plan parameter that satisfies the safety requirements:
\begin{equation}
    \label{eq:backup_trajopt}
    \param^{*} = \argmin_{\param} f(\state_{0}, \param)
    \st g(\state_0, \param) < 0
\end{equation}
where $f$ is an objective function and $g$ is the reachability-based safety constraint.
The objective function includes, but is not restricted, to the projection objective, meaning minimizing the difference between the proposed plan from the imitation learning policy and the parameterized backup plan. 
The solution of \eqref{eq:backup_trajopt} compactly represents the desired trajectory $\state(\timestep)$ whose initial state is $\state_{0}$ and safe throughout the horizon. 
By discretizing this trajectory and taking the desired state as the action representation, and using a low-level controller to track this trajectory, we formulate a backup policy as $\policy_{\backup}: \statespace \rightarrow \actionspace^{\backuphorizon}$, where $\backuphorizon$ is the horizon of the parameterized failsafe trajectory.
The framework is detailed in Algs.~\ref{algo:rail} and \ref{algo:reach_filter}.

\section{Continuous-Time Collision Checking}
\label{sec:cont_time_col_check}

We now detail our approach to verifying head plan safety to enable hard collision-avoidance constraints for IL.
This addresses two key challenges.
First, most IL policies output discrete states/actions, but collisions occur in continuous time.
Second, unlike in our past work on safe RL for mobile robots \cite{shao2021reachability,selim2022safe}, we cannot simply replace the IL output with parameterized plans (i.e., backup plans), because this would collapse the richness of the IL policy.

To solve these challenges, we propose a novel approach extension of our prior work \cite[Sec. VIII-A]{michaux2023can} to compute continuous-time swept volume for the head plan output by the IL motion policy.
First, we review forward occupancy, then explain our approach and implementation details.

\textbf{Manipulator Forward Occupancy.}
We compute $\forwardoccupancy$ in the following well-known way (c.f., \cite[Ch. 3.3]{lavalle2006planning}).
Suppose the robot has $n$ revolute joints and configuration $\config \in \configspace \subseteq (\unitcircle)^n$.
Suppose that the robot's \jth link has a local reference frame at the center of its joint with link $j-1$ and rotates about the local vector $\rotaxis_j$, and $j=0$ is the baselink/inertial frame.
For $\config \in \configspace$, let $\config_j$ denote its \jth element.
Then
$\rotmat_j^{j-1}(\config_j,\rotaxis_j) = \eye_{3\times 3} +
        \sin(\config_j)\rotaxis_j\supercross + (1-\cos(\config_j))(\rotaxis_j\supercross)^2$
is the rotation from frame $j-1$ to frame $j$ per  Rodrigues' rotation formula.
Suppose $\transvec_j^{j-1}$ is the origin of frame $j$ in frame $j-1$, and let $\link_j \subset \R^3$ represent the volume of the \jth link in the \jth reference frame.
Then link $j$'s forward occupancy is
\begin{align}
\forwardoccupancy_j(\config) = \{\transvec_j(\config)\}\oplus   
    \left(
        \rotmat_j(\config)\link_j
    \right) \subset \workspace,
    \label{eq: forward occupancy}
\end{align}
where $\rotmat_j(\config) = \prod_{i=1}^j\rotmat_j^{i-1}(\config_j)$
and $\transvec_j(\config) = \sum_{i=1}^j \rotmat_i(\config)\transvec_i^{i-1}$.
Letting $n = \dim(\configspace)$, the forward occupancy of the robot is
$\forwardoccupancy(\config) = \bigcup_{j=1}^{n} \forwardoccupancy_j(\config).$

\textbf{Method Overview.}
To proceed to overapproximate the rotation of a single joint using a box (Lem.~\ref{lem: box overapproximate arc}), which we use to overapproximate the joint's rotation matrices (Lem.~\ref{lem: rotation matrix set overapproximation}), and then the forward occupancy (Thm.~\ref{thm: forward occupancy overapproximation} and Cor.~\ref{cor: FO overapprox for trajectory}).

\begin{lem}\label{lem: box overapproximate arc}
Suppose a 1-D revolute joint travels counterclockwise (CCW) from an angle $\jointangle_1$ to $\jointangle_2 > \jointangle_1$.
Map this motion to the unit circle $\unitcircle$ as $\ucpoint_1 = (\cos(\jointangle_1),\sin(\jointangle_1))$ and
$\ucpoint_2 = (\cos(\jointangle_2),\sin(\jointangle_2))$.
Define
$\ucpoint_3 = \tfrac{1}{2}(\ucpoint_1 + \ucpoint_2)$ and
$\ucpoint_4 = (\cos(\tfrac{1}{2}(\jointangle_1+\jointangle_2),\sin(\tfrac{1}{2}(\jointangle_1+\jointangle_2))))$.
Let $\ucpoint_5 = \tfrac{1}{2}(\ucpoint_3 + \ucpoint_4)$.
Then, for every $\jointangle \in [\jointangle_1,\jointangle_2]$, we have
\begin{align*}
    (\cos(\jointangle),\sin(\jointangle)) \in 
    \boxset(\jointangle_1,\jointangle_2) :=
    \linesegment{\ucpoint_1\ucpoint_2}\oplus
    \linesegment{(\ucpoint_3-\ucpoint_5)(\ucpoint_4-\ucpoint_5)}.
\end{align*}
\end{lem}
\begin{proof}
Notice that $\ucpoint_3$ is halfway along the line segment between $\ucpoint_1$ and $\ucpoint_2$.
Similarly, $\ucpoint_4$ is the projection of $\ucpoint_3$ onto the unit circle CCW between $\ucpoint_1$ and $\ucpoint_2$.
Thus, by construction and CCW motion, $\zonotope(\vc{c},\vc{G})$ is a rectangle containing all four points and the arc from $\ucpoint_1$ to $\ucpoint_2$, as in Fig.~\ref{fig: arc overapproximation}.
\end{proof}

\begin{figure}[ht]
    \vspace*{-0.4cm}
    \centering
    \includegraphics[width=0.95\columnwidth]{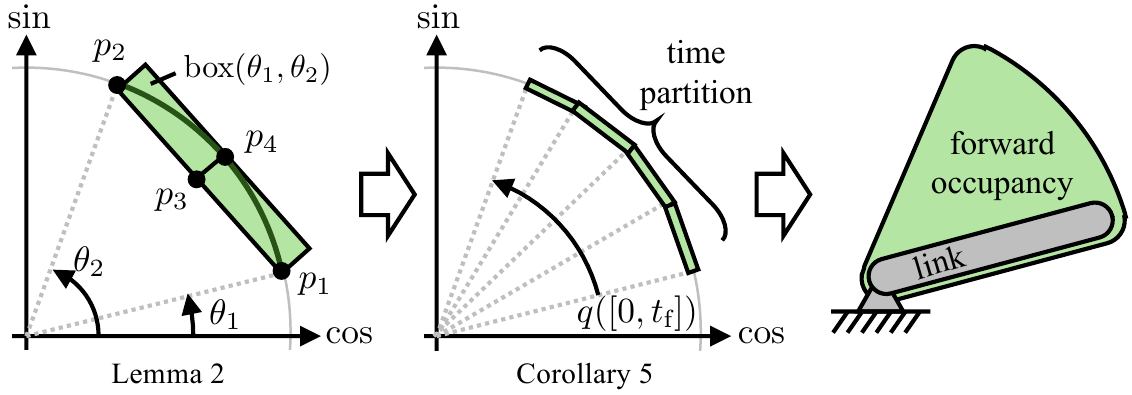}
    \vspace*{-0.2cm}
    \caption{We overapproximate the sines and cosines of the robot's joint angles along a trajectory to enable overapproximating the robot's swept volume.}
    \label{fig: arc overapproximation}
    \vspace*{-0.3cm}
\end{figure}

Next, we overapproximate the set of rotation matrices corresponding to a single revolute joint.
To enable this, we consider Rodrigues' rotation formula as a function
\begin{align}
    \rotfunc(\vc{p},\rotaxis) = \eye_{3\times 3} + p_2\rotaxis^\times + (1-p_1)(\rotaxis^\times)^2
\end{align}
for a point $\vc{p} = (p_1,p_2) \in \R^2$ and a rotation axis $\rotaxis \in \R^3$.
\begin{lem}\label{lem: rotation matrix set overapproximation}
    Consider a revolute joint with local rotation axis $\rotaxis \in \R^3$ rotating through an interval $[\theta^-,\theta^+]$.
    Then, for any rotation matrix $\rotmat$ corresponding to this rotation, we have $\rotmat \in \rotfunc(\boxset(\theta^-,\theta^+)),\rotaxis)$
\end{lem}
\begin{proof}
    This follows from Lemma \ref{lem: box overapproximate arc}, since Rodrigues' rotation formula is applied to every element of $\jointset$.
\end{proof}

Now we can overapproximate the swept volume of each link of the robot given bounded rotations of each joint:
\begin{thm}\label{thm: forward occupancy overapproximation}
Suppose each joint rotates through an interval $\jointset_j = [\theta_j^-,\theta_j^+]$.
Let $\jointset = \jointset_1\times\cdots\times\jointset_n$, and define
$
    \rotset_j(\jointset) = \prod_{i=1}^j \rotfunc(\boxset(\theta_j^-,\theta_j^+),\rotaxis_j)
$ and
$
    \transset_j(\jointset) = \sum_{i=1}^j \rotset_j(\jointset)\transvec_i^{i-1}.
$
Then, for any $\config \in \jointset$,
$
    \forwardoccupancy_j(\config) \subset \{\transset_j(\jointset)\}\oplus   
    \left(
       \rotset_j(\jointset,\rotaxis_j)\link_j
    \right),
$
where $\rotset_j(\jointset,\rotaxis_j)\link_j$ is a set product.
\end{thm}
\begin{proof}
    This follows directly from Lemma~\ref{lem: rotation matrix set overapproximation} and the forward occupancy definition \eqref{eq: forward occupancy}.
\end{proof}

Thm.~\ref{thm: forward occupancy overapproximation} would be very conservative if applied for all joints along an entire trajectory, since the joints are treated independent of time.
We mitigate this by partitioning time, which reduces the overapproximation as shown in Fig.~\ref{fig: arc overapproximation}:
\begin{cor}\label{cor: FO overapprox for trajectory}
    Consider a trajectory $\config: [0,t\final] \to \configspace$.
    Partition time into $[0,t\final] = \union_{i=1}^{m-1} [i\delta,(i+1)\delta]$ with $\delta = \frac{t\final}{m}$ and $m \in \N$.
    For each $i = 1,\cdots,m-1$, suppose for each \jth joint angle, we know its min and max values: $\config_j([i\delta,(i+1)\delta]) \in \jointset_{j,i} = [\config_{j,i}^-,\config_{j,i}^+]$.
    Then we construct $\jointset_i = \jointset_{1,i}\times\cdots\times\jointset_{n,i}$ and have
    \begin{align}\label{eq: FO overapprox for trajectory}
        \forwardoccupancy(\config([0,t\final])) \subset
            \union_{j=1}^n
            \union_{i=1}^{m-1} \{\transset_j(\jointset_i)\}\oplus   
            \left(
               \rotset_j(\jointset_i,\rotaxis_j)\link_j
            \right).
    \end{align}
\end{cor}
\begin{proof}
    This follows directly from Thm.~\ref{thm: forward occupancy overapproximation}.
\end{proof}

The challenge with using Cor.~\ref{cor: FO overapprox for trajectory} is bounding $\config(t)$ in each partition interval.
For the head plan, we know $\dot{\config}(t)$ at every $t$ and assume a maximum acceleration $\overline{\ddot{\config}}_j$, so we use the standard kinematic equation:
$\config_{j,i}^- \geq \config_{j}(i\delta) - \delta\dot{\config}_{j}(i\delta) - \tfrac{1}{2}\overline{\ddot{\config}}_j\delta^2$,
and $\config_{j,i}^+$ similarly.
The bound $\overline{\ddot{\config}}_j$ is found per joint from demonstration data and augmented by a small $\varepsilon > 0$.

\textbf{Implementation Details.}
Per Cor.~\ref{cor: FO overapprox for trajectory}, we can overapproximate the forward occupancy of the robot over an entire trajectory by considering a union over all links and all time intervals.
Let $\forwardoccapprox(\config([0,t\final]))$ denote the right-hand side of \eqref{eq: FO overapprox for trajectory}, so at runtime we check $\forwardoccapprox(\config([0,t\final])) \cap \obstacles = \emptyset$.
We implement this with polynomial zonotopes \cite{kochdumper2020sparse}.
All functions used in the approximation are analytic and thus overbounded per \cite{kochdumper2020sparse} (see \cite[\S III-C.5]{michaux2023can}), and polynomial zonotopes are closed under intersections.
We represent the unions in \eqref{eq: FO overapprox for trajectory} by storing a polynomial zonotope for each time interval and arm link separately.

\section{Results} \label{sec:results}


In this section, we seek to (a) confirm that RAIL does indeed enforce hard constraints on IL policies, and (b) understand the impact of these constraints.
We hypothesize that RAIL allows learning policies that leverage the long-horizon reasoning capability inherited from a base imitation learning policy while being robust to unseen danger using the short-horizon model-based reasoning of the backup planner.
It is worth noting, that our experiment does not assume that all safety definitions are present in the offline training dataset.
In addition, we consider that our model-based safety filter is privileged with the safety specification at inference runtime. 

\textbf{Baselines.} We consider two powerful imitation learning policies: Diffusion Policy~\cite{chi2023diffusion} and Action Chunking with Transformers (ACT)~\cite{zhao2023learning}. These policies are trained using offline datasets.
To understand the impact of \textit{soft constraints}, we also formulate a classifier-guided Diffusion Policy. This involves: (a) annotating the offline dataset with safe and unsafe labels based on whether the resulting state is safe after executing the action or not (b) training a classifier network to predict safety probability (c) following classifier-guided reverse sampling~\cite{dhariwal2021diffusion} with pre-trained diffusion policy network. 
Finally, we also evaluate the performance of a reachability-based backup planner~\cite{kousik2019safe, holmes2020armtd} as a model-based baseline with hard constraints.

\textbf{Evaluation metrics.}
We evaluate the success rate as reaching the goal in the target environment (Succ \%).
We measure the number of transitions that led to a collision between the agent and the obstacles and report it as a percentage of all the transitions (Col \%). 
Most importantly, we report the percentage of successful episodes where none of the safety conditions were violated throughout the trajectory (SSucc \%).
Lastly, to understand how efficient each agent is at completing tasks, we report the average number of timesteps of the trajectories (Horizon).


\begin{figure}[t]
    \centering
        \includegraphics[width=\columnwidth]{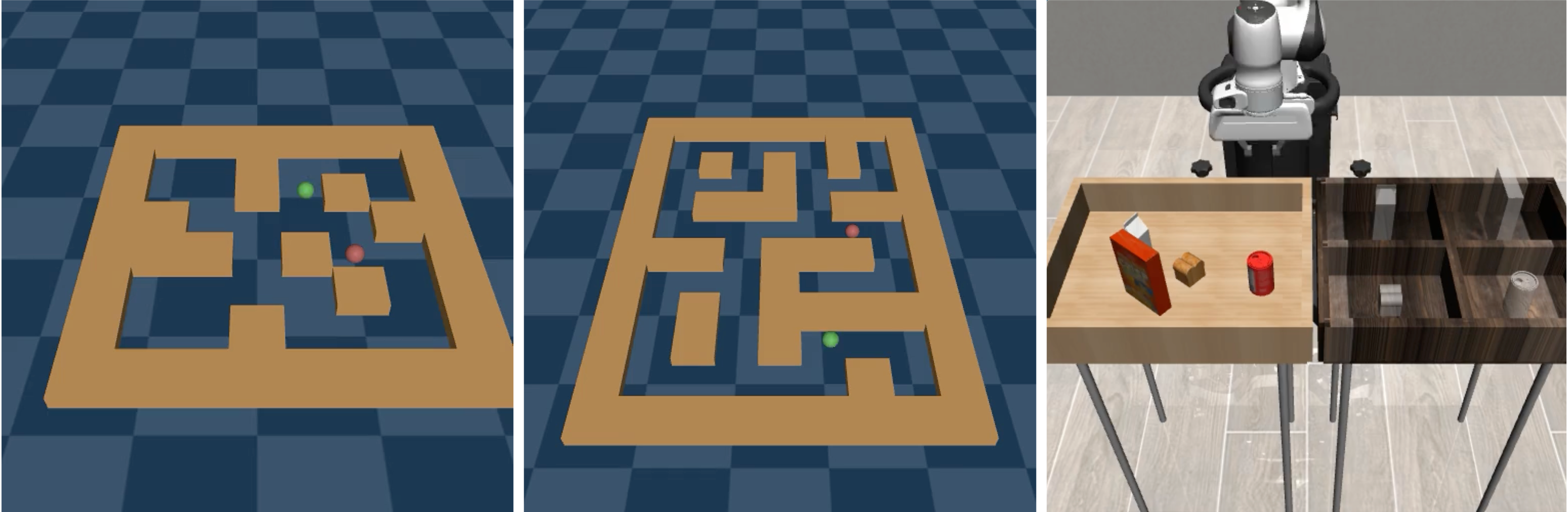}
    \caption{Evaluation tasks (left-to-right): \textbf{(a) Maze-Medium:} A point mass has to be taken to a goal in a maze without hitting the walls.
    \textbf{(b) Maze-Large:} Similar to Maze-Medium, but larger and more difficult to traverse.
    \textbf{(c) Can Pick-Place:} A can must be picked from a table and placed in a target bin.}
    \label{fig:all_envs}
\end{figure}

\subsection{Simulation: Safe Planning in Maze}


\begin{figure*}[ht]
    \centering
    \includegraphics[width=0.98\linewidth]{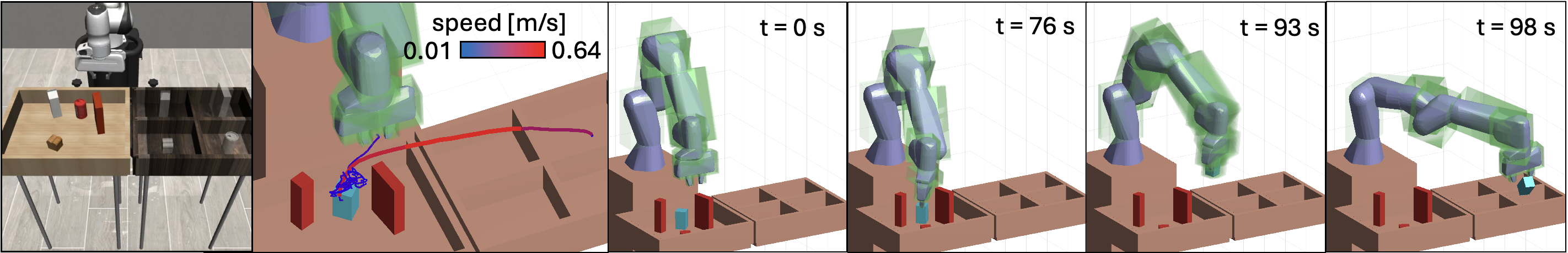}
    \caption{\textbf{(Leftmost)} Illustration of Can Pick-Place task: A can must be picked from a table and placed in a target bin. \textbf{(Rightward to Rightmost)} Timelapse of an episode involving navigating through a tight gap.
    A diffusion policy repeatedly proposed unsafe plans, until eventually RAIL validated a safe path to picking.
    The IL policy again proposed unsafe plans to move the picked object, until RAIL validated a safe path outwards to placing.}
    \label{fig: example trials}
\end{figure*}


\begin{table}[t]
\caption{Results on Maze Medium environment: We run each of the policies for 100 episodes and repeat for 3 seeds. We report the average values below.}
\label{tab:maze_medium_results}
\centering
\renewcommand{\arraystretch}{1.2}
\setlength{\tabcolsep}{4pt}
\begin{tabular}{|p{2cm}|c|c|c|c|}
\hline
\centering\textbf{Method} & \textbf{Succ. (\%)} & \textbf{Col. (\%)} & \textbf{SSucc. (\%)} & \textbf{Horizon} \\[0.5ex]
\hline
Diffusion Policy & \textbf{100.00} & 5.17 & 15.00 &  \textbf{278.39} \\
\hline
Backup Planner & 58.00 & \textbf{0.00} & 58.00 & 691.62 \\
\hline
Guided Diffusion & \textbf{100.00}& {4.64}&{18.00} & 284.84\\
\hline
\textbf{RAIL (ours)} & 97.00 & \textbf{0.00} & \textbf{97.00} & 309.77 \\
\hline
\end{tabular}
\end{table}

\begin{table}[t]
\caption{Results on Maze Large environment: We run each of the policies for 100 episodes and repeat for 3 seeds. We report the average values below.}
\label{tab:maze_large_results}
\centering
\renewcommand{\arraystretch}{1.2}
\setlength{\tabcolsep}{4pt}
\begin{tabular}{|p{2cm}|c|c|c|c|}
\hline
\centering\textbf{Method} & \textbf{Succ. (\%)} & \textbf{Col. (\%)} & \textbf{SSucc. (\%)} & \textbf{Horizon} \\[0.5ex]
\hline
Diffusion Policy & \textbf{100.00} & 5.66 & 15.00 & 470.48 \\
\hline
Backup Planner & 49.00 & \textbf{0.00} & 49.00 & 921.12\\
\hline
Guided Diffusion & \textbf{100.0}& {4.85}&{18.00} & \textbf{447.5}\\
\hline
\textbf{RAIL (Ours)} & 95.00 & \textbf{0.00} & \textbf{95.00} & 532.37 \\
\hline
\end{tabular}
\vspace*{-0.4cm}
\end{table}

\textbf{Setup.}
We evaluate our framework in the PointMaze environment~\cite{fu2020d4rl}, where the objective is for a double integrator mobile robot to navigate a maze to a target position.
We consider Medium and Large mazes as shown in~\autoref{fig:all_envs}.
The robot's state is 2-D position and velocity and its action is 2-D force.
Safety means avoiding collision with the maze walls and limiting velocity.

As our baseline, we train a diffusion policy on the offline D4RL dataset~\cite{fu2020d4rl} with an action trajectory prediction horizon of $\planninghorizon = 32$; at test time we execute a horizon of $\actionhorizon = 16$. 
The policy takes in the robot state as observations.
We adapt \cite{kousik2019safe} to design the backup planner.
We additionally compare our proposed method with a diffusion policy guided by a safety probability prediction classifier.

\begin{figure}[t]
    \centering
        \includegraphics[width=0.8\columnwidth]{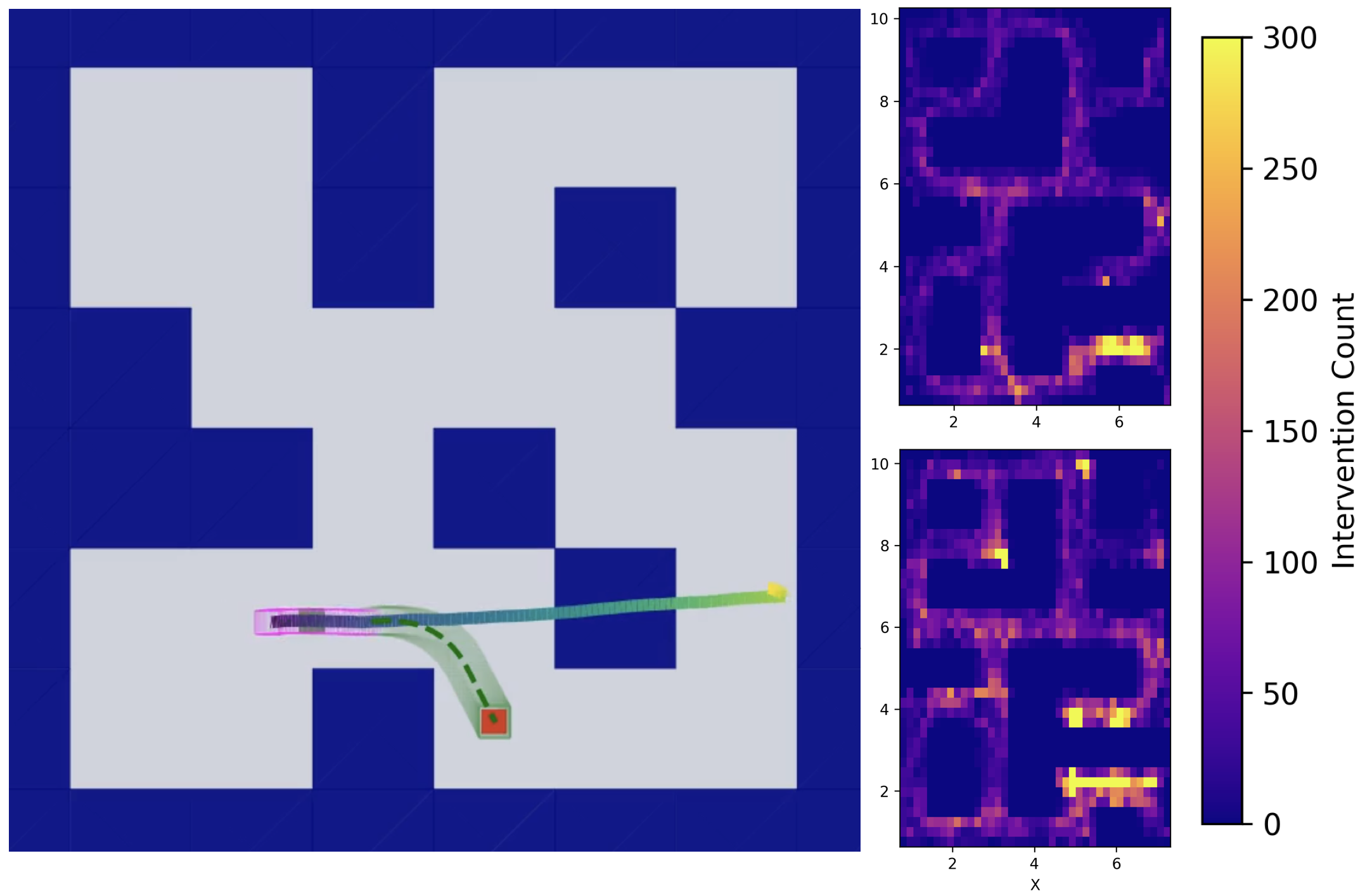}
    \caption{\textbf{(Left)} Visual representation of RAIL in the Maze-Medium task. The diffusion policy's proposed plan (gradient line) is validated by checking the reachable set of head plans (magenta) and the existence of a backup plan (green dotted line), with its reachable set (green tube). \textbf{(Right)} Heatmap comparing the number of safety interventions by RAIL in the Maze-Large task over 100 episodes with random initializations, comparing diffusion policies at Epoch 1900 \textbf{(Top)} and Epoch 50 \textbf{(Bottom}).}
    \label{fig:maze_internal}
    \vspace*{-0.5cm}
\end{figure}



\textbf{Results and Discussion.}
We see similar results in both the Medium (\autoref{tab:maze_medium_results}) and Large mazes (\autoref{tab:maze_large_results}).
It is evident that even if diffusion policy can solve the task, it collides with the walls.
In this experiment, we specifically try to make this more obvious by deliberately sabotaging our policy with a new safety specification at runtime (i.e., the offline dataset includes the trajectory colliding with obstacles).
On the other hand, the safe model-based backup planner~\cite{kousik2019safe} guarantees safety but shows a low success rate due to lack of long-horizon reasoning capability.
Guided diffusion implicitly informs and guides the action sampling from diffusion policy and hence leads to decrease the number of collisions, but cannot guarantee constraint satisfaction.
Our proposed method, RAIL leverages the conservative only backup planner when diffusion policy suggests unsafe actions, enabling it to achieve higher task success despite hard constraints.

We further analyze how IL policy performance affects integration with RAIL by evaluating trained models across different epochs.
The number of safety interventions notably decreases in later epochs \textit{despite no explicit safety training signal} (see \autoref{fig:maze_internal}).





\subsection{Simulation: Safe Planning in Pick-and-place}

\textbf{Setup.}
We now consider a more complex task derived from robomimic~\cite{robomimic2021} that involves a single-arm manipulator picking up a can and placing it in a target bin. 
Safety means avoiding collisions between the robot and the environment while obeying joint position and velocity constraints.
To enable picking, the collision constraint between the gripper and the can is ignored.
Additionally, the robot's behavior can be broken down into reaching, picking, and placing; so, we impose phase-specific constraints, especially in placing, where collisions between the can and the environment are unsafe.
We represent all objects as bounding boxes (see \autoref{fig: example trials}).

For all experiments, we adapt ARMTD~\cite{holmes2020armtd} to design a backup planner. 
We train Diffusion Policy~\cite{chi2023diffusion} and ACT~\cite{zhao2023learning} on 184 safe, successful demonstrations collected via robomimic.
The observation space is end effector pose, gripper position, and objects' poses.
We evaluate both policies with our proposed RAIL safety filter.
For all policies, we sample action sequences with a horizon of $\actionhorizon = (8, 16)$, $\planninghorizon = 32$.
For diffusion policy, we use a DDIM~\cite{song2020denoising} scheduler with 100 reverse denoising timesteps for sampling.



\textbf{Results and Discussion.}
We summarize results in~\autoref{tab:pickplace_results}.
We see that the diffusion policy and ACT without RAIL have high success.
This is because the task requires picking in clutter, so the policy learns to push the clutter aside, leading to a high collision rate; the SSucc numbers confirm this.
On the other hand, with RAIL, the diffusion policy is able to perform many more tasks both safely and successfully, indicating that many of the scenarios were possible to achieve without collisions.
We also see that RAIL incurs nearly twice the horizon of the unfiltered diffusion policy, indicating the policy and safety layer often oppose each other (see \autoref{fig: example trials}, \autoref{tab:pickplace_results}).

We make a key surprising finding by comparing the safe-and-success rate across seeds of diffusion policy and ACT, shown in \autoref{tab:pickplace_results_ssucc_delta}.
All of the policies considered had at least one seed (typically not the highest-performing) that significantly improved its SSucc when integrated with RAIL.
In other words, \textit{enforcing hard constraints can sometimes increase the performance of an IL policy}.
As suggested by \autoref{fig: example trials}, this is likely because (a) the safety filter forces the policy to behave more cautiously to carry out fine-grained tasks in clutter, where the constraints are always close to being active, and (b) lower-performing seeds may oppose the safety filter less often.



\begin{table}[h]
\caption{Results on Pick-Place environment: We run each of the policies for 50 episodes with 5 different seeds. We report the average values below.}
\label{tab:pickplace_results}
\centering
\renewcommand{\arraystretch}{1.2}
\setlength{\tabcolsep}{4pt}
\begin{tabular}{|p{2.3cm}|c|c|c|c|}
\hline
\centering\textbf{Method} & \textbf{Succ. (\%)} & \textbf{Col. (\%)} & \textbf{SSucc. (\%)} & \textbf{Horizon} \\[0.5ex]
\hline
Diffusion Policy & \textbf{78.00} & 27.20 & 58.00 & \textbf{723.08} \\
\hline
ACT & 74.00 & 35.39 & 62.00 & 1434.98\\
\hline
RAIL + DP & 68.00 & \textbf{0.00} & \textbf{68.00} & 1385.34 \\
\hline
RAIL + ACT & 58.00 & \textbf{0.00} & 58.00 & 1294.18 \\
\hline
\end{tabular}
\vspace*{-0.4cm}
\end{table}

\begin{table}[h]
\caption{Comparison of Increase in SSucc Rate per Policy Class. For each policy class, we run 8 policies of different hyperparameters and evaluate the increase of SSucc rate between IL and RAIL}
\label{tab:pickplace_results_ssucc_delta}
\centering
\renewcommand{\arraystretch}{1.2}
\setlength{\tabcolsep}{4pt}
\begin{tabular}{|p{2.3cm}|c|c|c|}
\hline
\centering\textbf{Method} & \textbf{Min. (\%)} & \textbf{Max. (\%)} & \textbf{Mean. (\%)} \\[0.5ex]
\hline
DP ($\actionhorizon = 8$)  & \textbf{4.00}   & \textbf{16.00} & \textbf{9.25} \\
\hline
DP ($\actionhorizon = 16$) & -14.00 & 10.00 & 2.25 \\
\hline
ACT ($\actionhorizon = 8$)  & -4.00   & 8.00 & 2.00 \\
\hline
ACT ($\actionhorizon = 16$) & -18.00 & 4.00  & -7.20 \\
\hline
\end{tabular}
\vspace*{-0.4cm}
\end{table}


\section{Real World Evaluation}


Finally, we demonstrate RAIL with a diffusion policy in the real world on a Franka Panda robot arm shown in~\autoref{fig:teaser_figure}.
This indicates our method can enforce hard constraints in real time and under sensing uncertainty.

\textbf{Setup.}
The goal is to pick and place the target object at a goal while avoiding delicate wine glasses and stacked blocks (represented for the safety filter as bounding boxes and detected with AprilTags~\cite{apriltag} via an Azure Kinect DK RGBD camera.)
Other safety specifications are as in simulation.
We train a diffusion policy on 75 teleoperated, safe demonstrations (collected with Gello~\cite{wu2023gello}).
The policy's observation space is the arm's joint and gripper positions and objects' poses.
We choose the desired joint angle as action representation, and joint impedance controller as low-level controller.
ARMTD~\cite{holmes2020armtd} serves as the backup planner.
Running the safety filter and rolling out the diffusion policy in parallel requires predicting future robot/environment state; for this, we directly use the desired joint angles from the policy with an empirically determined latency horizon to minimize the prediction error.
We assume all obstacles stay static.

\textbf{Results and Discussion.}
We found that the diffusion policy with RAIL solves the task safely as expected (see ~\autoref{fig:teaser_figure}), while the vanilla diffusion policy failed to comply with the safety specifications.
This indicates our safety filter runs in real-time ($0.42 \pm 0.05$ seconds per plan); however, note the main focus of this paper is not maximially efficient implementation, which we leave to future work.

\section{Conclusion}
This paper presented Reachability Assisted Imitation Learning (RAIL), a novel framework that combines deep imitation learning with a reachability-based safety filter to enforce hard constraints.
Simulated and real-world evaluations show that RAIL can indeed enforce hard constraints (0\% collision rate) on state-of-the-art IL policies in real-time.
Excitingly, in some cases RAIL \textit{increases} policy performance, indicating hard constraints are not always a negative tradeoff.

\textbf{Limitations.}
RAIL has two main limitations. 
First, it assumes access to predicted observations after action execution, which can be a bottleneck for scaling to policies conditioned on high-dimensional observations (e.g., images, languages) where making accurate predictions is challenging. 
Second, applying RAIL can result in longer task completion times due to potential conflicts between IL policy and the safety filters. 
A possible future direction is to incorporate safety filters directly into the IL policy planning process to generate inherently safer plans and minimize these conflicts.

\renewcommand{\bibfont}{\normalfont\footnotesize}
{\renewcommand{\markboth}[2]{}
\printbibliography}

\end{document}